\documentclass{tADR2e}
\usepackage{algorithm}

\newcommand{\argmin}{\mathop{\rm arg~min}\limits}
\newtheorem{assumption}{Assumption}
\newtheorem{theorem}{Theorem}
\newtheorem{lemma}{Lemma}
\newtheorem{corollary}{Corollary}
\usepackage{subfigure}
\usepackage{algorithm}
\usepackage{algorithmic}
\usepackage{color}
\renewcommand{\algorithmicrequire}{\textbf{Inputs:}}
\renewcommand{\algorithmicensure}{\textbf{Initialization:}}
\renewcommand{\algorithmicendwhile}{\textbf{until}}
\renewcommand{\algorithmicif}{\textbf{for}}
\renewcommand{\algorithmicthen}{\textbf{do}}

\begin{document}



\title{Development of global optimal coverage control using multiple aerial robots}

\author{K. Shibata$^{1\dag}$\thanks{$^{\dag}$Corresponding author. Cloud Informatics Research-Domain, Toyota Central R$\&$D Labs., Inc., 41-1, Yokomichi, Nagakute, Aichi, Japan. Email: kshibata@mosk.tytlabs.co.jp. Tel: +81-561-71-7849. 
\vspace{6pt}} , T. Miyano$^{1}$ and T. Jimbo$^{1}$ \\\vspace{6pt}  $^{1}${\rm{Cloud Informatics Research-Domain, Toyota Central R$\&$D Labs., Inc., Nagakute, Aichi, Japan}}
}

\maketitle

\begin{abstract}
Coverage control has been widely used for constructing mobile sensor network such as for environmental monitoring, and one of the most commonly used methods is the Lloyd algorithm based on Voronoi partitions. However, when this method is used, the result sometimes converges to a local optimum. To overcome this problem, game theoretic coverage control has been proposed and found to be capable of stochastically deriving the optimal deployment. From a practical point of view, however, it is necessary to make the result converge to the global optimum deterministically. In this paper, we propose a global optimal coverage control along with collision avoidance in continuous space that ensures multiple sensors can deterministically and smoothly move to the global optimal deployment. This approach consists of a cut-in algorithm based on neighborhood importance of measurement and a modified potential method for collision avoidance. The effectiveness of the proposed algorithm has been confirmed through numerous simulations and some experiments using multiple aerial robots.

\medskip

\end{abstract}

\section{Introduction}

In recent years, mobile sensor networks have been utilized in numerous industrial fields such as for environmental monitoring, infrastructure inspection, land surveying, search and rescue. In these application scenarios, sensor deployment problems need to be resolved so that multiple autonomous mobile sensors can be deployed to target positions in specific coverage regions where they are needed to collect valuable data from the environment.

Sensor deployment problems have been discussed in numerous studies through both centralized and distributed control methods. For centralized control, deployment methods using evolutionary computation were proposed in \cite{1, 2, 3}. In these studies, the total coverage performance and the energy consumption are optimized by using calculation techniques such as genetic algorithms and particle swarm optimization. However, the feasibility of real-time operation becomes severely limited if the number of mobile sensors increases significantly. Moreover, in terms of fault tolerance, the system should be capable of operating even if some mobile sensors run out of fuel or malfunction.

In contrast, distributed control of sensors has advantages in the computational cost and resilience. Among those methods, the Lloyd algorithm discussed in \cite{4, 5} is well known. 
This method is based on Voronoi partitions and is widely used considering real environmental applications in \cite{6, 7, 8}. 
However, deployments obtained by Voronoi-based algorithms sometimes converge to local optimums because each agent tends to maximize the coverage rate solely within its defined region. 
Thus, depending on initial deployment or changes in target distribution patterns, this can lead to inefficient deployments with respect to coverage rates and convergence speed. 
In \cite{9}, it was reported that if multiple agents are restricted to the use of the information within their assigned Voronoi regions, the resulting coverage performance will not necessarily be better than that produced by a single agent.

To overcome local optimality issues in distributed control of sensors, game theoretic approaches were proposed in \cite{10, 11}. In these methods, agents self-organize their positions according to gain in the partitioned grid. As a result, the sensor deployment can stochastically converge to the global optimal deployment while considering obstacles and speed differences among robots. In \cite{12}, another game theoretic control was proposed for unknown environments containing obstacles that has proved to be capable of stochastically converging to the Nash equilibrium. Moreover, based on the control method discussed in \cite{12}, literature \cite{13} proposed another learning algorithm that proved capable of stochastically converging to the global optimal deployment. However, from a practical point of view, it is desirable that the results deterministically converge to the global optimum.

In an effort to guarantee the global optimality, another Voronoi-based algorithm is proposed in \cite{14} in which each agent shares the importance magnitude in neighboring regions defined by their Voronoi partitions. Moreover, collision avoidance, which forces robots to move on a discretized grid, was considered, and multi-color mass game numerical simulations showed that this proposed algorithm provided better coverage performance than a conventional algorithm.

Moreover, since the number of mobile sensors is often less than that of the targets in actual situations, dynamic routing in partitioned Voronoi regions is proposed to handle this problem in \cite{15}. In \cite{16}, persistent coverage control is proposed as another approach to address the same problem. In this control method, a dynamic distribution density function is designed in the given region. In other words, by decreasing the density within the sensor range, mobile sensors tend to gravitate toward the other target positions with higher density.

In this paper, we propose a global optimal coverage control with collision avoidance in continuous space that ensures multiple robots can deterministically and smoothly move to global optimal positions. In the case that the number of agents is less than that of targets, the importance of measurement to be already covered is deleted or decreased such as in \cite{16}. Therefore, this paper deals with the case of same number of agents and targets in both simulations and experiments.

The remainder of this paper is organized as follows. Section 2 describes the coverage problem considered in this paper, while section 3 introduces the mathematical formula for coverage control, and then proposes a global optimal coverage control algorithm that considers collision avoidance. Section 4 shows results obtained through numerical simulations under randomly arranged targets, while section 5 introduces the constructed experimental system and shows experimental results obtained using multiple unmanned aerial robots. Finally, we conclude the paper in section 6.

\section{Problem statement}
In this section, we introduce the coverage problem considered in this paper. In the given coverage region $Q$, agent $i\in \nu:=\{1,\cdots,n\}$ moves to capture targets $l\in T:=\{1,\cdots,n_t\}$ as shown in figure \ref{fig1}. Here, we denote the positions of agent $i\in \nu$ and target $l\in T$ as $\textit{\textbf{x}}_i$ and $\boldsymbol{\mu}_l$, respectively. Let an importance of measurement at the point $\textit{\textbf{q}}$ be expressed as $\phi(\textit{\textbf{q}})$. In this paper, $\phi(\textit{\textbf{q}})$ is provided to agent $i\in \nu$ as given information. Based on $\phi(\textit{\textbf{q}})$, each agent $i\in \nu$ communicates with other agents and attempts to move to a target position in a distributed manner. To quantify the coverage problem, we define $\phi(\textit{\textbf{q}})$ as point mass given by $\phi(\textit{\textbf{q}})=1$ if $\textit{\textbf{q}}=\boldsymbol{\mu}_l$, otherwise $\phi(\textit{\textbf{q}})=0$.

Next, we define a visual sensor model for agent $i\in \nu$. Since the measurable visual sensor range is normally limited, we consider the coverage problem in which one target is covered by one agent. Let denote the visual sensor range of agent $i\in \nu$ as $D(\textit{\textbf{x}}_i)$ and the coverage rate of target $l\in T$ is given by
\begin{eqnarray}
P_l=
  \begin{cases}
   1 \ (\boldsymbol{\mu}_l \in D(\textit{\textbf{x}}_i), ^{\exists}i \in \nu)\\
   \\
   0 \ (\boldsymbol{\mu}_l \notin D(\textit{\textbf{x}}_i), ^{\forall}i \in \nu)\\
  \end{cases}.
\label{eq1}
\end{eqnarray}
\begin{figure}
\begin{center}
\includegraphics[width=10cm]{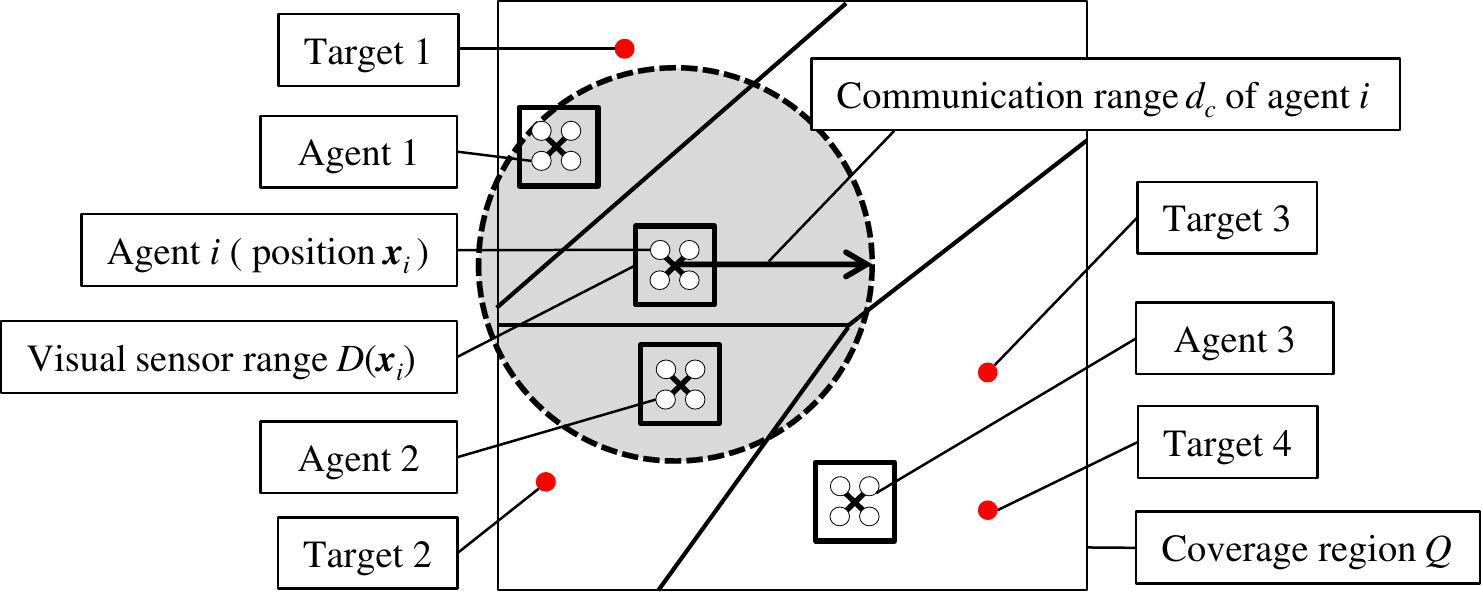}
\caption{Coverage problem in two-dimensional space. The square, dashed circle, colored dot, bold line and gray area represent the visual sensor range of each agent, communication range of agent, target, Voronoi partitions and the neighborhood $\mathcal{R}_i$ of agent $i$, respectively. 
We can calculate $C_i$ as $C_i=\{\emptyset\}$ since no target exists within the region of agent $i$. The index set of agents with which agent $i$ can communicate is given by $\mathcal{N}_i=\{1,2\}$.
Using (\ref{eq8}), the coverage rates of target 1 and 2 which agent $i$ obtains from its neighboring agents 1 and 2 are given by $P_1^{\mathcal{N}_i}=0$ and $P_2^{\mathcal{N}_i}=0$ since target 1 and 2 are uncovered.
Similarly, the coverage rates of target 3 and 4 which agent $i$ obtains are given by $P_3^{\mathcal{N}_i}=null$ and $P_4^{\mathcal{N}_i}=null$ since agent 1 cannot obtain these coverage rates from agent 3.
The target index set which is memorized by agent $i$ as an uncovered target and the target index set whose coverage rate is not obtained by agent $i$ are calculated as $V_i=\{ 1,2\}$ and $W_i=\{ 3,4\}$ using (\ref{eq11}) and (\ref{eq12}), respectively.}
\label{fig1}
\end{center}
\end{figure}

Using (\ref{eq1}), the total coverage rate $P_{\rm{cov}}$ can be calculated by
\begin{eqnarray}
P_{\rm{cov}}=\frac{1}{n_t}\sum^{n_t}_{l=1}P_l\ge \frac{1}{n_t}\sum^{n}_{i=1}\sum_{l\in C_i}P_l=P_{\rm{cov}}^L,
\label{eq2}
\end{eqnarray}
where $C_i$ is the target index set assigned to agent $i\in \nu$ discussed in section 3. The lower bound $P^L_{\rm{cov}}$ is equal to $P_{\rm{cov}}$ when the target $l\in C_i$ is necessarily covered by agent $i$. If $P^L_{\rm{cov}}$ converges to one, $P_{\rm{cov}}$ also converges to one since $P_{\rm{cov}}\le 1$.
Therefore, we solve the sensor deployment problem by optimizing $P^L_{\rm{cov}}$ in (\ref{eq2}).

Finally, we define a communication model for each agent. Since each agent can communicate with other agents within a circle with the certain radius of $d_c$, the neighborhood $\mathcal{R}_i$ of agent $i\in \nu$ is expressed as 
\begin{eqnarray}
\mathcal{R}_i=\{ \textit{\textbf{q}} \in Q:\| \textit{\textbf{q}}-\textit{\textbf{x}}_i\| \le d_c\}, \nonumber
\end{eqnarray}
where $\| \bullet \|$ represents Euclidean norm of a vector. Based on $\mathcal{R}_i$, the index set of agents with which agent $i\in \nu$ can communicate is given by
\begin{eqnarray}
\mathcal{N}_i=\{j\in \nu:\textit{\textbf{x}}_j\in \mathcal{R}_i \}.
\label{eq3}
\end{eqnarray}
Note that $\mathcal{R}_i$ and $\mathcal{N}_i$ depend on time.

\section{Control design}

In this section, we propose a global optimal coverage control to derive global optimal solutions. To escape local optimal solutions, a cut-in algorithm based on neighborhood importance of measurement is integrated into the Lloyd algorithm. 
The global optimality of our algorithm is proved without considering collision avoidance among different agents. Moreover, our algorithm allows some agents to move out of their defined regions while taking collision avoidance into consideration.

\subsection{Coverage control}
This section introduces the formulation of the Lloyd algorithm proposed in \cite{4}. In this algorithm, the coverage region $Q$ is divided into Voronoi regions defined by perpendicular bisectors established between neighboring agents. Based on this method, we define the target index set assigned to agent $i\in \nu$ given by
\begin{eqnarray}
C_i:=\{l\in T:\| \boldsymbol{\mu}_l-\textit{\textbf{x}}_i \| \le \| \boldsymbol{\mu}_l-\textit{\textbf{x}}_j \|, \ j\in \mathcal{N}_i \}.
\label{eq4}
\end{eqnarray}

Using (\ref{eq4}), $P_{\rm{cov}}^L$ in (\ref{eq2}) can be calculated. Note that the target index set $C_i$ is calculated by the deployment one time step prior to the current time step and is assumed to be invariant during each time step. Based on $C_i$, agent $i\in \nu$ determines its next reference position given by
\begin{eqnarray}
\textit{\textbf{x}}_i^{\ast}=\boldsymbol{\mu}_{l_{\ast}}\ \rm{s.t.}\it{}\ l_{\ast}=\argmin_{l\in C_i}\| \boldsymbol{\mu}_l-\textit{\textbf{x}}_i \|,
\label{eq5}
\end{eqnarray}
where $\textit{\textbf{x}}_i^{\ast}$ is the next reference position of agent $i\in \nu$. Let the dynamics of agent $i\in \nu$ be given by
\begin{eqnarray}
\dot{\textit{\textbf{x}}}_i=\textit{\textbf{u}}_i,
\label{eq6}
\end{eqnarray}
where $\textit{\textbf{u}}_i$ is the control input of agent $i\in \nu$. Using (\ref{eq5}), the control input $\textit{\textbf{u}}_i$ is calculated by
\begin{eqnarray}
\textit{\textbf{u}}_i=k_i(\textit{\textbf{x}}_i^{\ast}-\textit{\textbf{x}}_i),
\label{eq7}
\end{eqnarray}
where $k_i$ is a positive constant determined by the maximum speed of agent $i\in \nu$.

\subsection{Global optimal coverage control}
The deployment obtained by the Lloyd algorithm sometimes converges to a local optimum because each agent can only move within its defined region. 

To derive a global optimal solution, we first integrate cut-in algorithm in \cite{14} into the Voronoi-based algorithm. As additional information, agent $i\in \nu$ shares $P_l$ in (\ref{eq1}) with its neighborhood agent $j\in \mathcal{N}_i$. 
At each time step $t$, agent $i\in \nu$ obtains the coverage rate of target $l\in C_j$ from its neighboring agent $j\in \mathcal{N}_i$ given by
\begin{eqnarray}
P_l^{\mathcal{N}_i}(t)=
  \begin{cases}
   P_l(t)\ \ \rm{if}\ \it{}l\in C_j \\
   \\
   null \ \ \ \rm{otherwise}\\
  \end{cases},
\label{eq8}
\end{eqnarray}
where $null$ represents an empty value.
The coverage rate of target $l\in T$ which is memorized by agent $i\in \nu$ is updated by
\begin{eqnarray}
P_l^{i}(t+\Delta t)\leftarrow
  \begin{cases}
   P_l^{\mathcal{N}_i}(t)\ \ \rm{if}\ \it{}P_l^{\mathcal{N}_i}(t)\neq null \\
   \\
   P_l^{i}(t)\ \ \ \ \rm{otherwise}\\
  \end{cases},
\label{eq9}
\end{eqnarray}
where $\Delta t$ represents the time step. In the proposed algorithm, the next reference position of agent $i\in \nu$ is given by
\begin{subequations}
\begin{align}
\textit{\textbf{x}}_i^{\ast}&=\boldsymbol{\mu}_{l_{\ast}}\ \rm{s.t.}\it{} \ l_{\ast}=\argmin_{l\in C_i} \|\boldsymbol{\mu}_l-\textit{\textbf{x}}_i\| \ \rm{if} \ \it{}C_i\neq \emptyset, \label{eq10a} \\
\textit{\textbf{x}}_i^{\ast}&=\boldsymbol{\mu}_{l_{\ast}}\ \rm{s.t.}\it{} \ l_{\ast}=\argmin_{l\in V_i} \|\boldsymbol{\mu}_l-\textit{\textbf{x}}_i\| \ \rm{if} \ \it{}C_i= \emptyset \ \rm{and}\it{}\ V_i\notin \emptyset, \label{eq10b} \\
\textit{\textbf{x}}_i^{\ast}&=\boldsymbol{\mu}_{l_{\ast}}\ \rm{s.t.}\it{} \ l_{\ast}=\argmin_{l\in W_i} \|\boldsymbol{\mu}_l-\textit{\textbf{x}}_i\| \ \rm{if} \ \it{}C_i= \emptyset \ \rm{and}\it{}\ V_i= \emptyset \label{eq10c},
\end{align}
\label{eq10}
\end{subequations}
\begin{eqnarray}
V_i&=&\{ l\in T:\ P_l^i(t)=0 \}, \label{eq11} \\
W_i&=&\{ l\in T:\ P_l^i(t)=null\} \label{eq12},
\end{eqnarray}
where $\emptyset$, $V_i$ and $W_i$ represent an empty set, the target index set which is memorized by agent $i\in \nu$ as an uncovered target and the target index set whose coverage rate is not obtained by agent $i\in \nu$, respectively.

Next, we will prove that the proposed coverage control algorithm shown in \textbf{Algorithm 1} makes the coverage rate in (\ref{eq2}) exponentially converge to one. To accomplish this, the following assumption is made.

\begin{assumption}
\rm{It} takes longer for $\boldsymbol{\mu}_{l_{\ast}}$ in (\ref{eq10a}), (\ref{eq10b}) and (\ref{eq10c}) to change sufficiently than dynamics (\ref{eq6}) with (\ref{eq7}).
\end{assumption}
Suppose that Assumption 1 holds, then the following are trivial.

\begin{corollary}
\rm{When agent} $i\in \nu$ moves using input (\ref{eq7}) with (\ref{eq10}), the position of agent $i\in \nu$, $\textit{\textbf{x}}_i$, exponentially converges to $\textit{\textbf{x}}_i^{\ast}$ in (\ref{eq10}) from Assumption 1 and the positive constant $k_i$ in (\ref{eq7}).
\end{corollary}
\begin{corollary}
\rm{Let consider} the case where more than or equal to one agent keep moving to target $l\in T$. When these agents move using input (\ref{eq7}) with (\ref{eq10}), one of the agents become closest from target $l\in T$ and target $l\in T$ is allocated to the closest agent using (\ref{eq4}).
\end{corollary}

From the above corollaries, the following lemma is obtained.

\begin{lemma}
\rm{When agent} $i\in \nu$ moves using input (\ref{eq7}) with (\ref{eq10}), the position of agent $i\in \nu$, $\textit{\textbf{x}}_i$, exponentially converges to one of the targets for $n=n_t$.
\end{lemma}

\begin{proof}
After sufficient time has elapsed, one of the targets is allocated to only one agent from Corollary 2. 
If one of the targets is allocated to agent $i\in \nu$, the position of agent $i\in \nu$ exponentially converges to the position of the target from Corollary 1.
Meanwhile, when agent $i\in \nu$ is deprived of the next target by another agent, agent $i\in \nu$ selects another uncovered target as its next reference position and moves to the reference position. 
After repeating this process, for $n=n_t$, the position of agent $i\in \nu$ exponentially converges to one of the targets.
\end{proof}

Finally, we can derive the following theorem to guarantee the global optimal deployment.

\begin{theorem}
\rm{When agent} $i\in \nu$ moves to the reference position in (\ref{eq10}), the total coverage rate, $P_{\rm{cov}}$, exponentially converges to the global optimum value one.
\end{theorem}
\begin{proof}
For $n=n_t$, the position of each agent exponentially converges to one of the targets in $Q$ from Lemma 1.
Similarly, for $n>n_t$, the positions of $n_t$ robots converge to one of the targets.
Thus, the coverage rate $P_{\rm{cov}}$ converges to one since $P_{\rm{cov}}^L\rightarrow 1$.
Moreover, for $n<n_t$, we change the importance of measurement which is covered by agents into zero using the idea in \cite{16}. 
After repeating this process, the number of uncovered targets becomes equal to or less than the number of agents. As a result, we can solve the problem in the same way as $n=n_t$ or $n>n_t$.
\end{proof}

From Theorem  1, in this paper, we consider only the case of the same number of agents and targets in both simulations and experiments in sections 4 and 5.

\begin{algorithm}[!tp]
\caption{: Proposed algorithm for agent $i$}
\begin{algorithmic}[1]
\REQUIRE $Q$, $\nu$, $T$, $D$, $d_c$, $k_i$, $t_L$ ($t_L$: the last time step)
\ENSURE $\textit{\textbf{x}}_i$ $(i=1,\cdots,n)$, $\boldsymbol{\mu}_l$ $(l=1,\cdots,n_t)$
\IF{$t=1,\cdots,t_L$}
\STATE Check $\mathcal{N}_i(t)$ using (\ref{eq3})
\STATE Calculate $C_i(t)$ using (\ref{eq4})
\STATE Obtain $P_l^{\mathcal{N}_i}(t)$ using (\ref{eq8})
\STATE Update $P_l^i(t)$ using (\ref{eq9})
\STATE Calculate $V_i$ and $W_i$ using (\ref{eq11}) and (\ref{eq12})
\STATE Calculate $\textit{\textbf{u}}_i(t)$ using (\ref{eq7}) and (\ref{eq10})
\STATE \textbf{return} $\textit{\textbf{u}}_i(t)$
\ENDIF
\end{algorithmic}
\end{algorithm}

\subsection{Collision avoidance}
In \cite{14}, the collision avoidance method is also considered in discrete space. In this method, each agent is placed within a divided pixel that is two times larger than the size of the agent, and is then passed through the space among other agents. However, the ability of agents to move smoothly toward the target positions can be lost because agent movements are limited in discrete space when collision avoidance is active. Therefore, based on the method outlined in \cite{17}, we propose a modified potential method that balances the global optimal deployment, the collision safety, and the trajectory smoothness.

Since collision avoidance requires maintaining safe distances between the different active agents, we add the repulsive force for each agent and modify the dynamics of (\ref{eq6}) as follows: 
\begin{eqnarray*}
\dot{\textit{\textbf{x}}}_i=\textit{\textbf{u}}_i+\delta \textit{\textbf{u}}_i,
\end{eqnarray*}
where $\delta \textit{\textbf{u}}_i$ is the vector which modifies the control input $\textit{\textbf{u}}_i$ to avoid collision and is given by
\begin{eqnarray}
\delta \textit{\textbf{u}}_i&=&\sum^n_{j=1}L_{ij} k_d^i\exp \left( -\|\textit{\textbf{x}}_i-\textit{\textbf{x}}_j\|^2\right)\frac{\textit{\textbf{x}}_i-\textit{\textbf{x}}_j}{\|\textit{\textbf{x}}_i-\textit{\textbf{x}}_j \|}, \label{eq13} \\
k_d^i&=&
  \begin{cases}
   K_sK_d \ \ \rm{if}\it{}\ C_i=\emptyset \\
   \\
   K_d \ \ \ \ \ \ \rm{otherwise}\\
  \end{cases}, \label{eq14}
\end{eqnarray}
where $K_d$ and $K_s$ represent positive values to avoid collision between different agents and positive values to adjust repulsive force of the agent which moves beyond its Voronoi region, respectively.
$L_{ij}$ is equal to one if agent $j$ exists within certain distance $d_k$ from agent $i$, otherwise it is equal to zero.
As discussed in \cite{18}, an agent can be surrounded by other agents, thereby resulting in deadlock when potential methods are applied. 

To avoid such problems, the repulsive force is reduced by decreasing $K_s$ in (\ref{eq14}) if $C_i=\emptyset$. 
If $K_s$ asymptotically approaches to zero, agent $i$ which owns no target within its Voronoi region exponentially converges to the position of targets because the movement of the agent is not affected by the repulsive force in (\ref{eq13}). 
On the other hand, if the level of collision safety is ensured, we should increase $K_s$ to some extent.
In the tuning, we first adjusted $K_d$ to maintain the minimum distance between different agents.
Moreover, we searched the minimum value of $K_s$ for fixed value $K_d$ to avoid the deadlock while maintaining the level of collision safety.

In this paper, both collision safety and 100 $\%$ coverage are not theoretically guaranteed. However, through numerous simulations and some experiments, we demonstrate 100 $\%$ coverage can be achieved while minimizing collision by adjusting $K_s$ and $K_d$. 

\section{Simulations}
In Theorem 1 to guarantee global optimal coverage, collision avoidance among agents is not considered. Therefore, in this section, we confirm the global optimal deployment with collision safety of the proposed algorithm through large-scale simulations with 100 trials.

\subsection{Simulation conditions}
Simulation conditions are shown in table 1. In the simulations, the number of agents matches that of targets. The deployments of targets are randomly set for each trial by generating each target from $20\times20$ grids in the coverage region given by
$$
q\in Q:=\{ (x,y)\mid -10\le x \le 10,\ -10\le y \le 10\}.
\eqno{}
$$
The size of the agents is set to that of the robots used in our experiment and the collision detection distance is set considering the robot size. 

\begin{table}[!bp]
\tbl{Simulation conditions}
{\begin{tabular}[l]{@{}lccc}\toprule
  Variable & Symbol & Value\\
\colrule
  Number of agents & $n$ & 100  \\
  Number of targets & $n_t$ & 100  \\
  Coverage region size [m]$\times$[m] & - & 20$\times$20  \\  
  Time step[s] & $\Delta t$ & 0.02  \\  
  Max speed of agent $i\in \nu$ [m/s] & - & 5.0  \\  
  Visual sensor range of agent $i\in \nu$ [m]$\times$[m] & $D({\textit{\textbf{x}}_i})$ & 1.0$\times$1.0  \\  
  Communication range [m] & $d_c$ & 10.0  \\  
  Agent size [m]$\times$[m] & - & 0.3$\times$0.3  \\  
  Collision detection distance [m] & - & 0.3  \\  
  Distance to activate collision avoidance [m] & $d_k$ & 0.55  \\     
  Parameters for collision avoidance & $K_d$ & 8.0$\times10^2$  \\ 
  Parameters for collision avoidance & $K_s$ & 0.35  \\                    
\botrule
\end{tabular}}
\label{symbols}
\end{table}

\subsection{Simulation results}
Portions of 100 simulation results from the Lloyd algorithm and our proposed algorithm are shown in figures \ref{fig2} and \ref{fig3}, respectively. 
From figure \ref{fig2b}, we can see that when the Lloyd algorithm is used, some agents cannot move to uncovered targets beyond their Voronoi regions since they own no target within their Voronoi regions.
As a result, a lot of targets remain uncovered. 
Meanwhile, figures \ref{fig3b} and \ref{fig3c} show that when our proposed algorithm is used, agents can move beyond their Voronoi regions even if they own no target within their Voronoi regions.
As a result, all the targets are covered by all the deployed agents as shown in figure \ref{fig3d}.

In order to confirm the global optimality of our proposed algorithm, the total coverage rate results are shown in figure \ref{fig4}.
From the results, it can be seen that total coverage rates of the proposed algorithm converge to 100 $\%$ in all cases as shown in figure \ref{fig4b}, whereas the rates of the Lloyd algorithm do not reach 100 $\%$ in all cases as shown in figure \ref{fig4a}.

\begin{figure}[!tp]
\centering
\subfigure[$t=0.0$ s]{
\includegraphics[width=7cm]{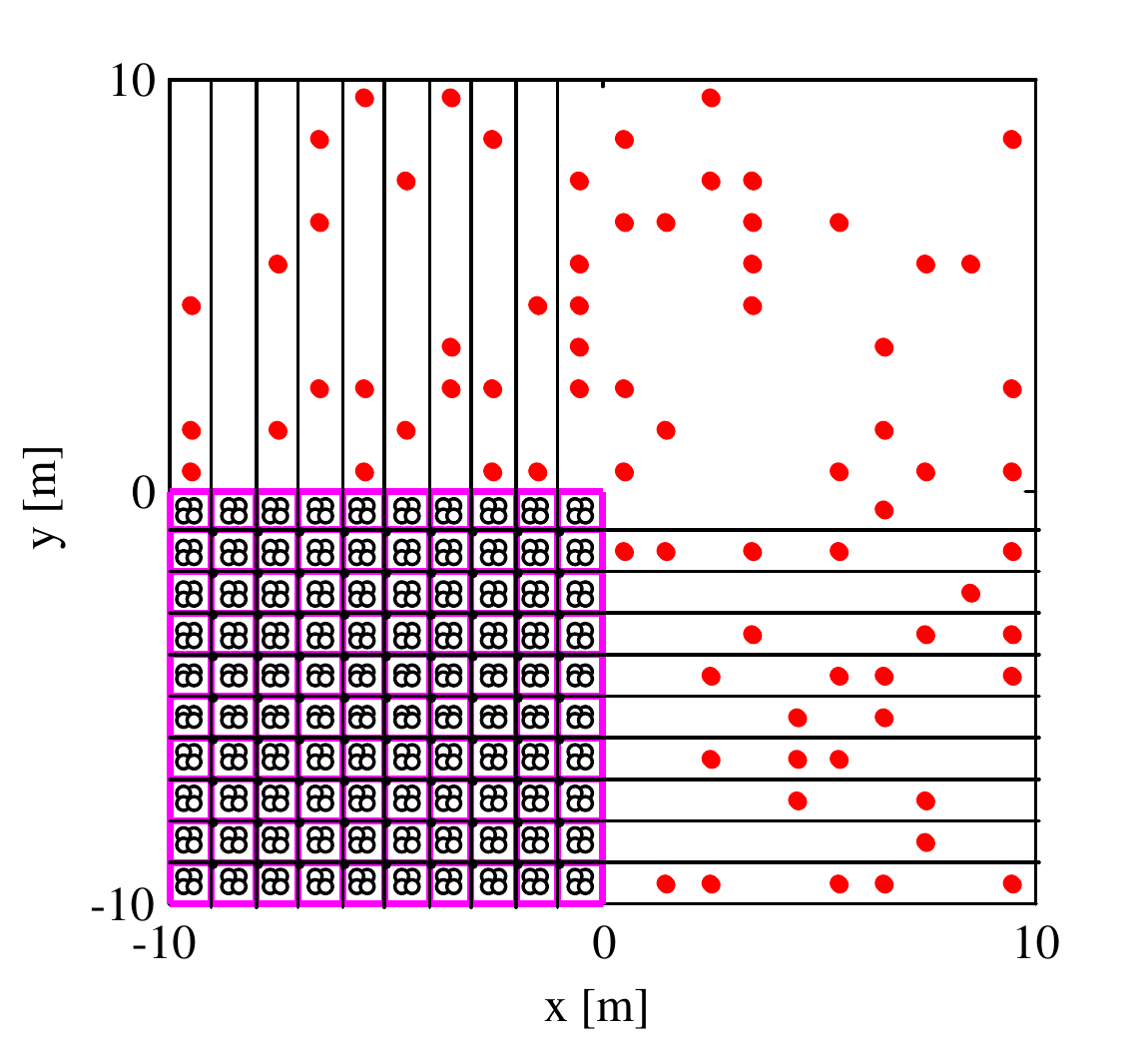}
\label{fig2a}}
\subfigure[$t=10.0$ s]{
\includegraphics[width=7cm]{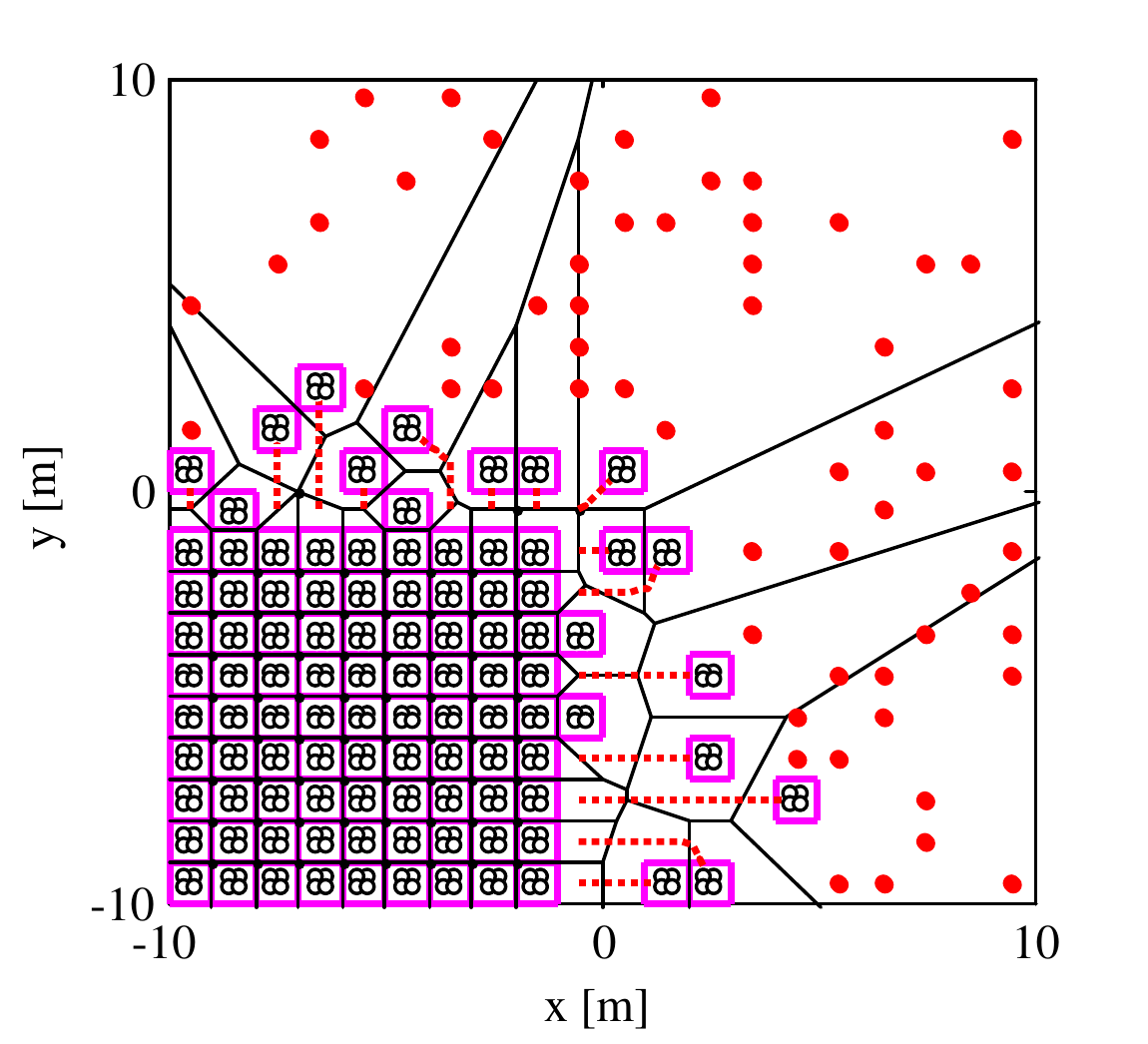}
\label{fig2b}}
\caption{Deployment changes of 100 agents by applying the Lloyd algorithm. Solid lines, square frames, dots, and dashed lines represent Voronoi partitions, visual sensor ranges, target positions, and trajectories of agents up to the time, respectively.}
\label{fig2}
\end{figure}
 
\begin{figure}[!bp]
\centering
\subfigure[$t=0.0$ s]{
\includegraphics[width=7cm]{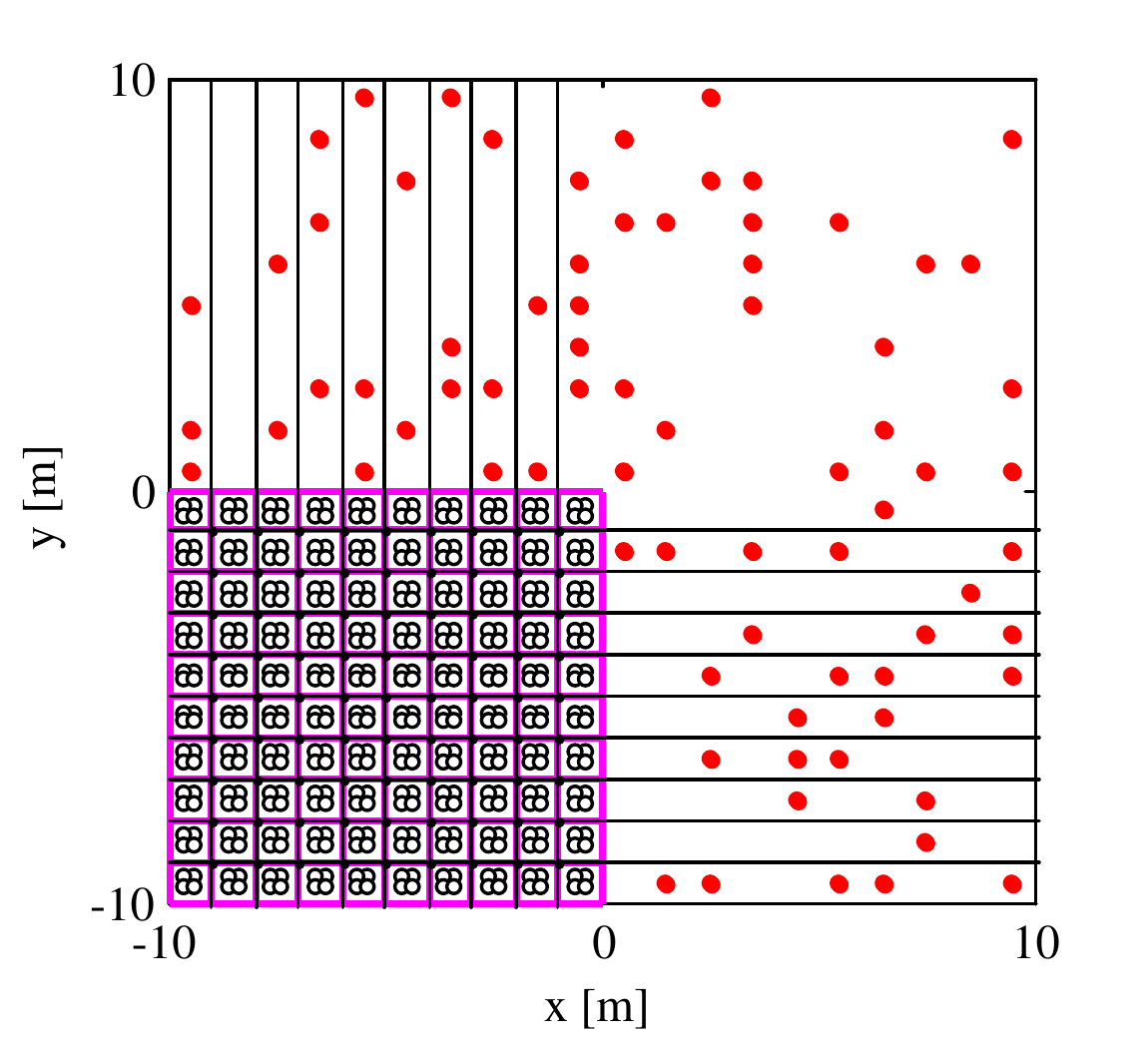}
\label{fig3a}}
\subfigure[$t=2.0$ s]{
\includegraphics[width=7cm]{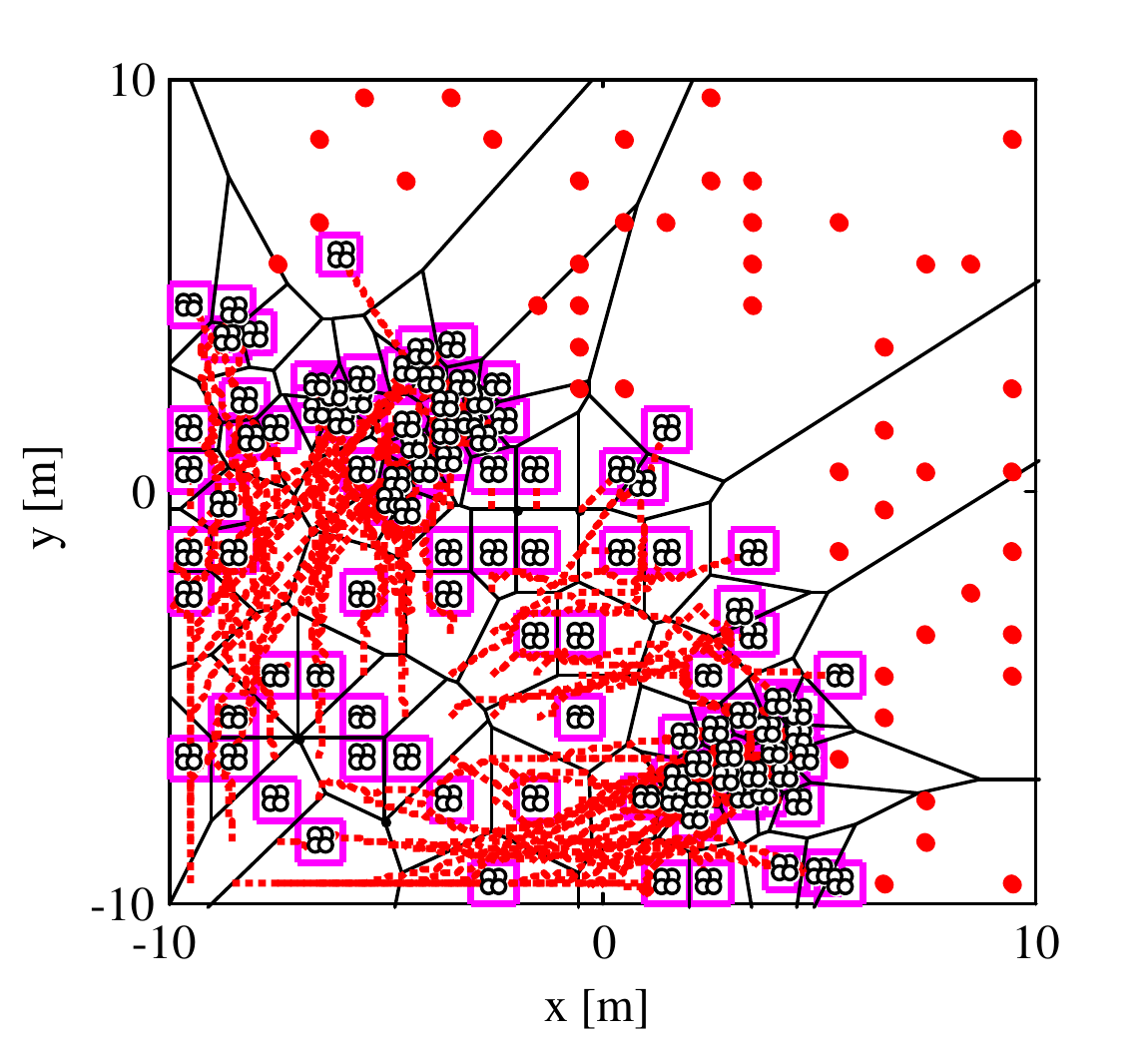}
\label{fig3b}}
\subfigure[$t=6.0$ s]{
\includegraphics[width=7cm]{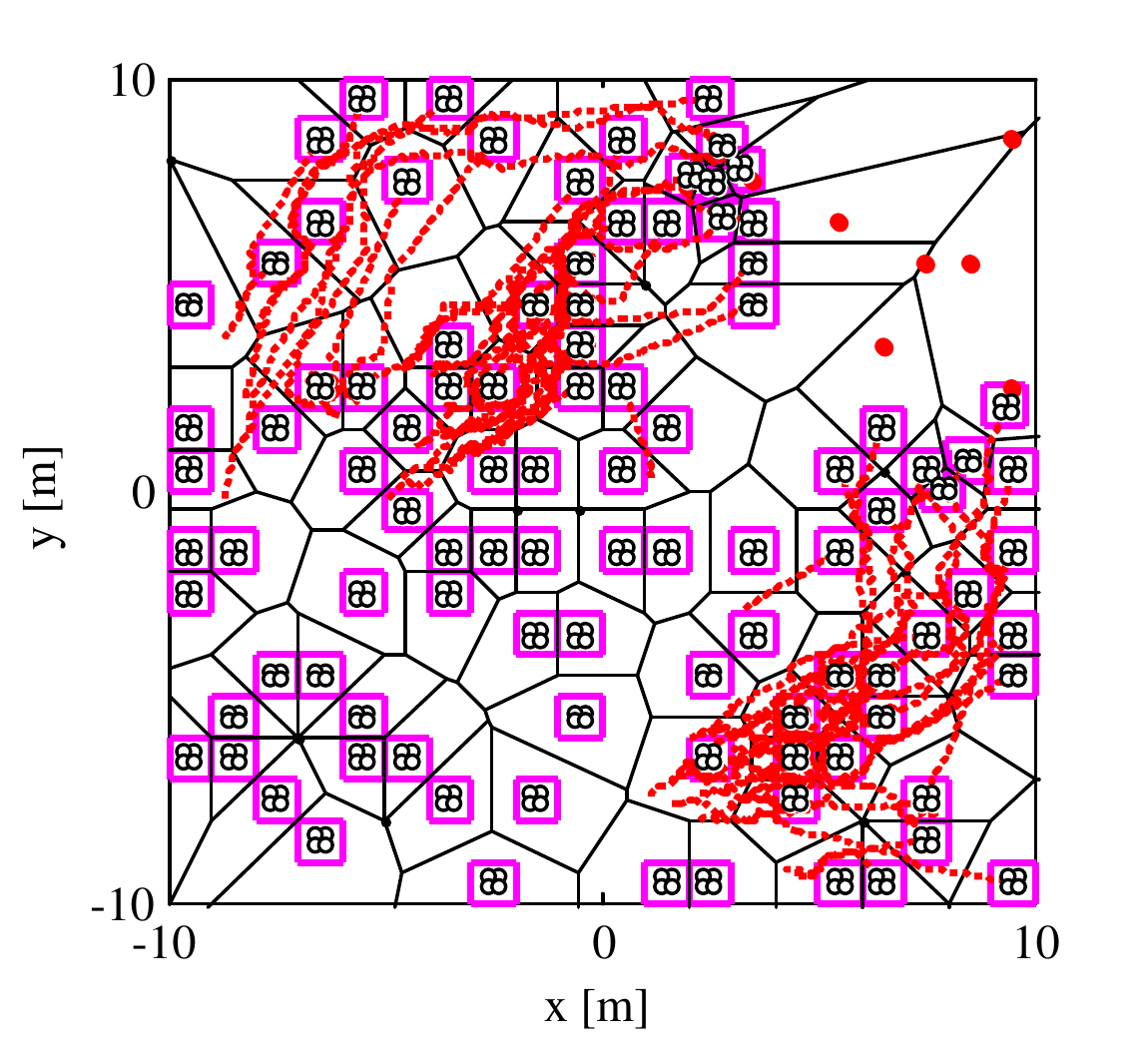}
\label{fig3c}}
\subfigure[$t=10.0$ s]{
\includegraphics[width=7cm]{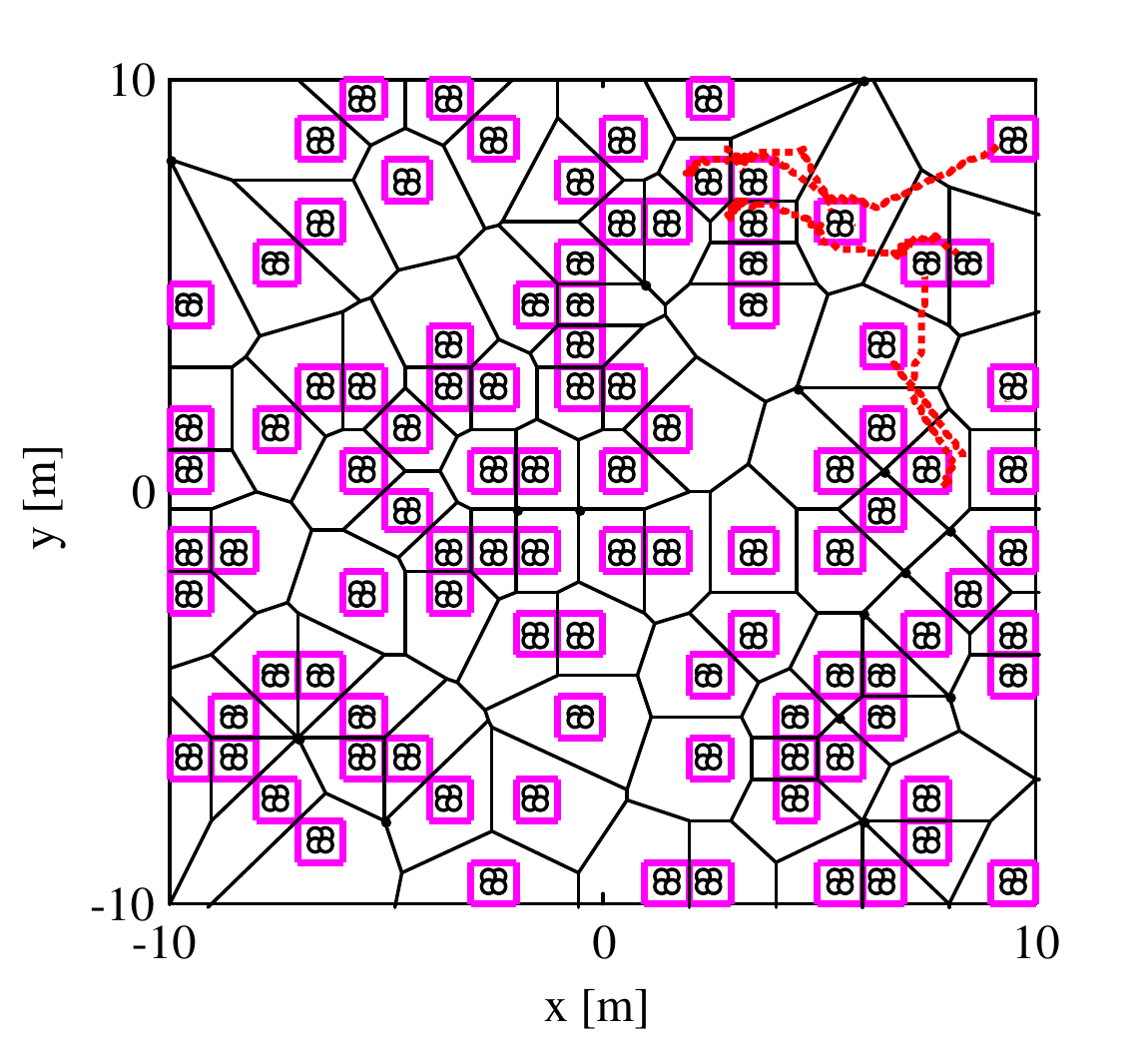}
\label{fig3d}}
\caption{Deployment changes of 100 agents by applying the proposed algorithm.
Solid lines, square frames, dots, and dashed lines represent Voronoi partitions, visual sensor ranges, target positions, and trajectories of agents up to the time, respectively.}
\label{fig3}
\end{figure}
 
\begin{figure}[!tp]
\centering
\subfigure[Lloyd algorithm]{
\includegraphics[width=7cm]{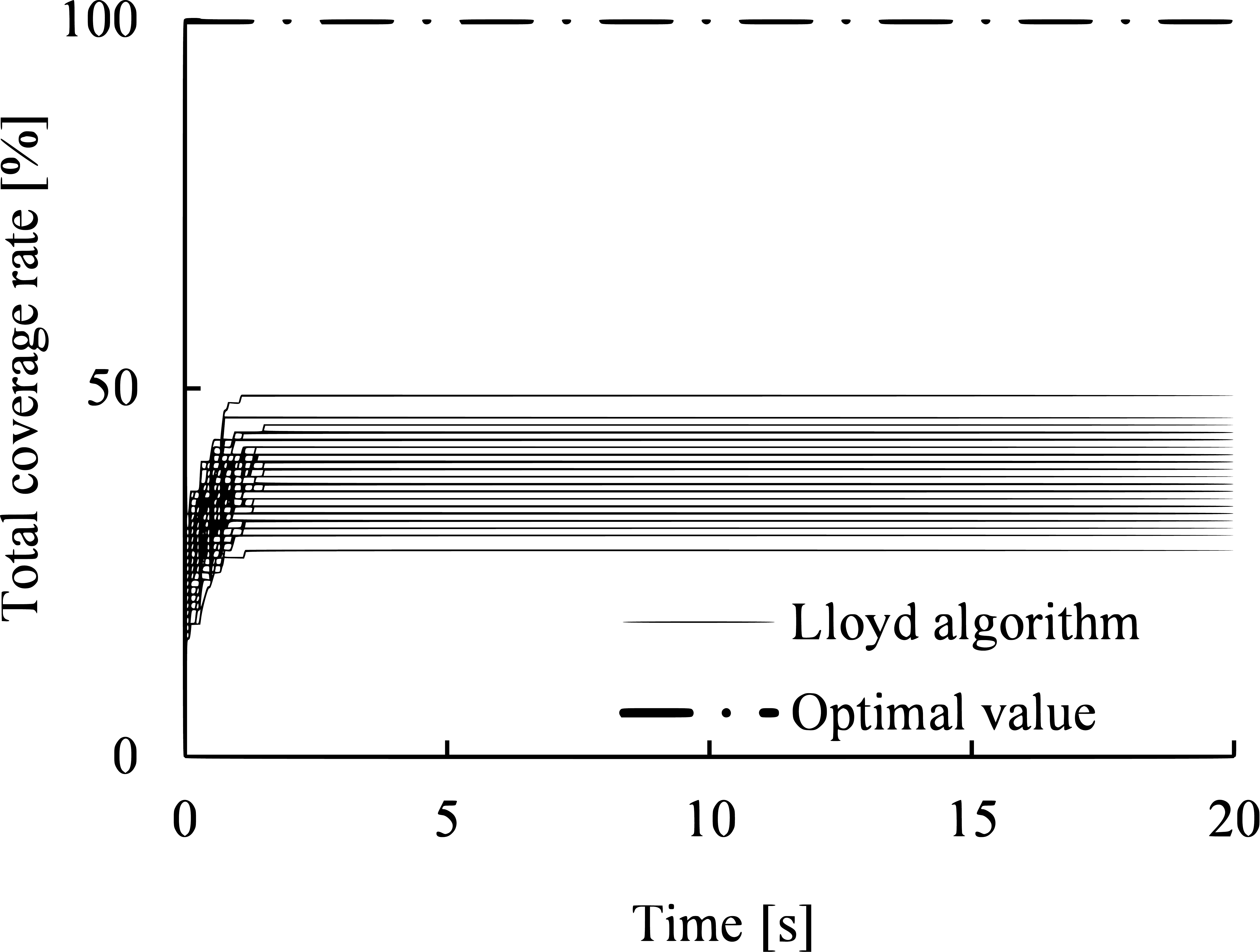}
\label{fig4a}}
\subfigure[Proposed algorithm]{
\includegraphics[width=7cm]{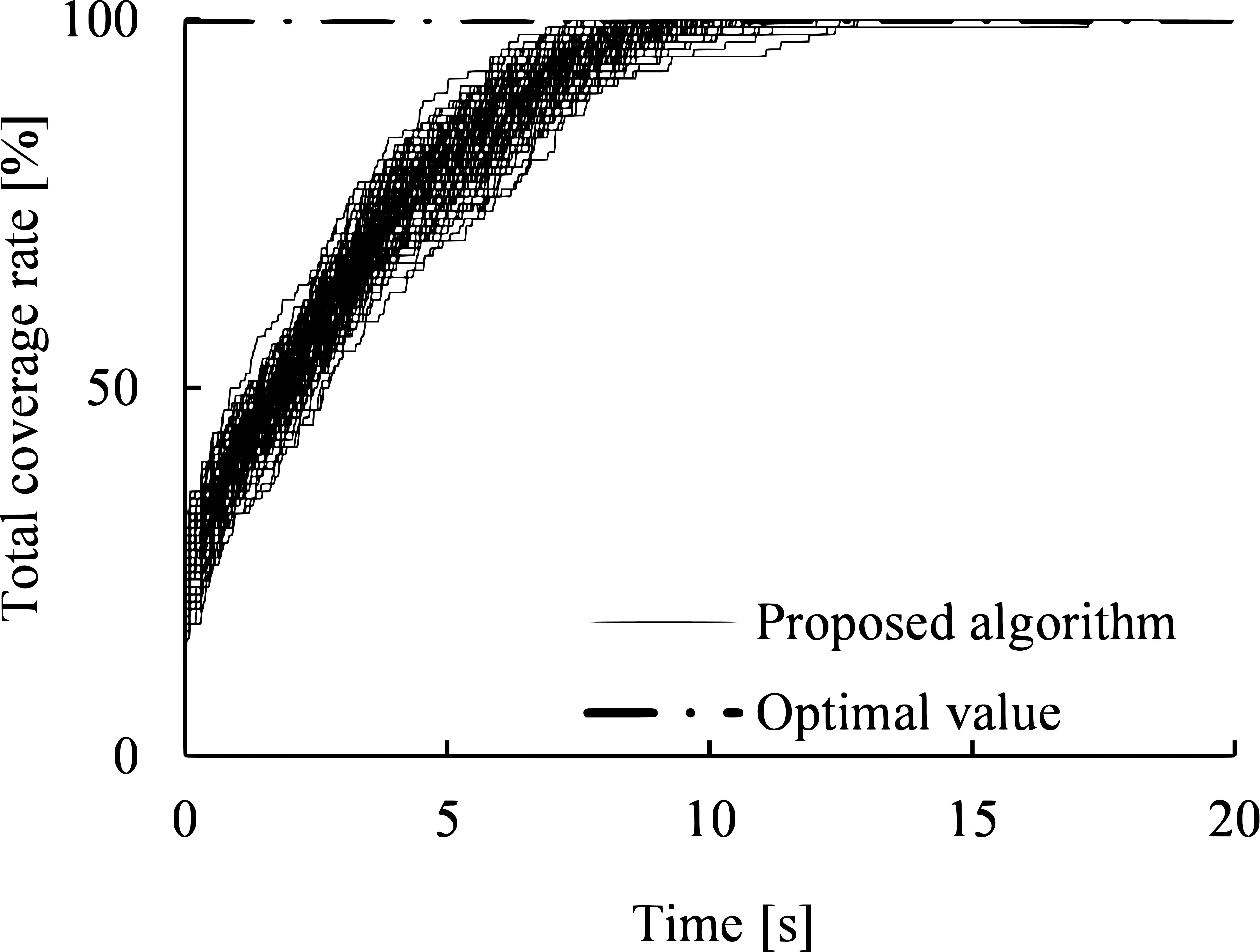}
\label{fig4b}}
\caption{Total coverage rate comparisons between the Lloyd and our proposed algorithms covering 100 simulations.}
\label{fig4}
\end{figure}

\begin{figure}[!tp]
\centering
\subfigure[Lloyd algorithm]{
\includegraphics[width=7cm]{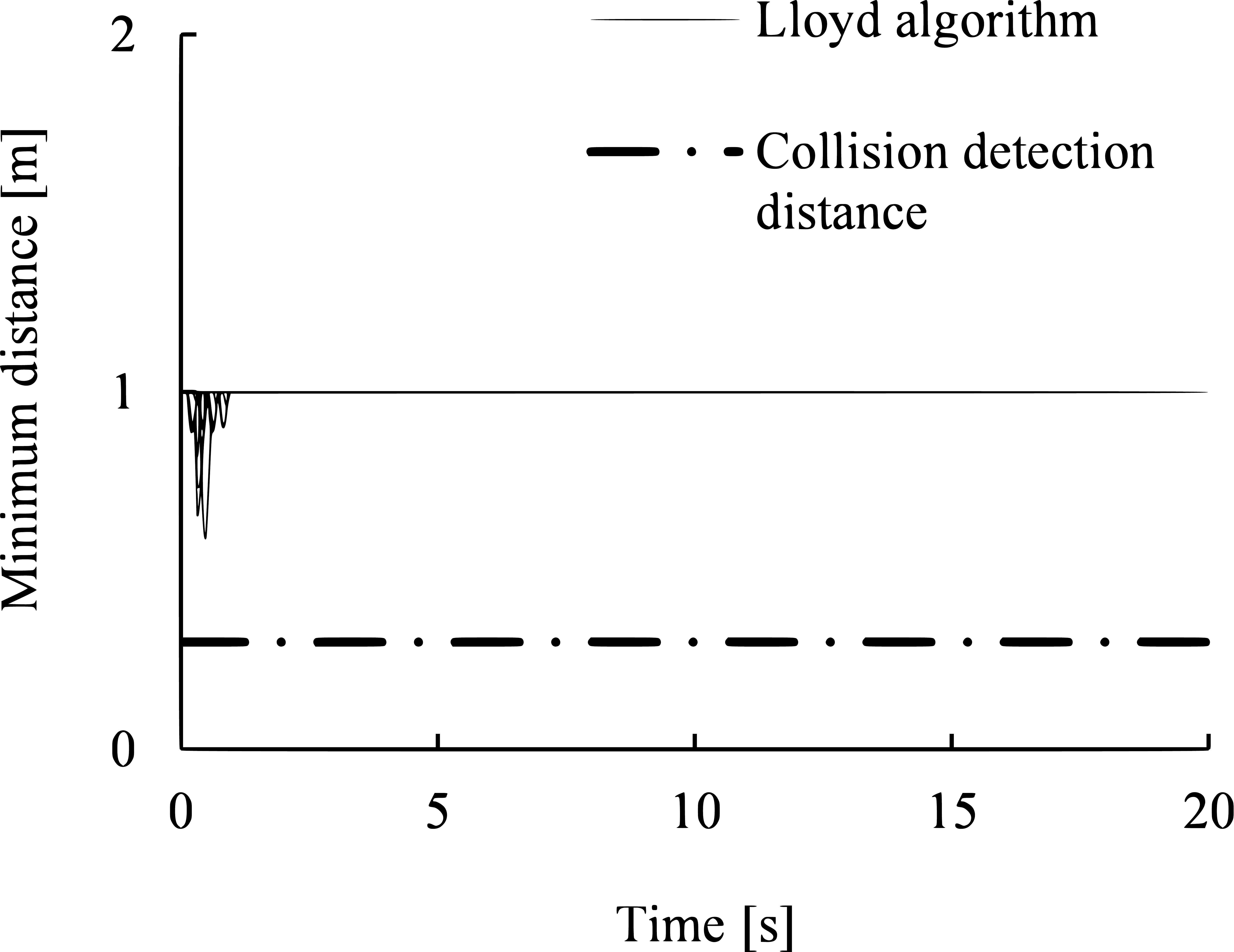}
\label{fig5a}}
\subfigure[Proposed algorithm]{
\includegraphics[width=7cm]{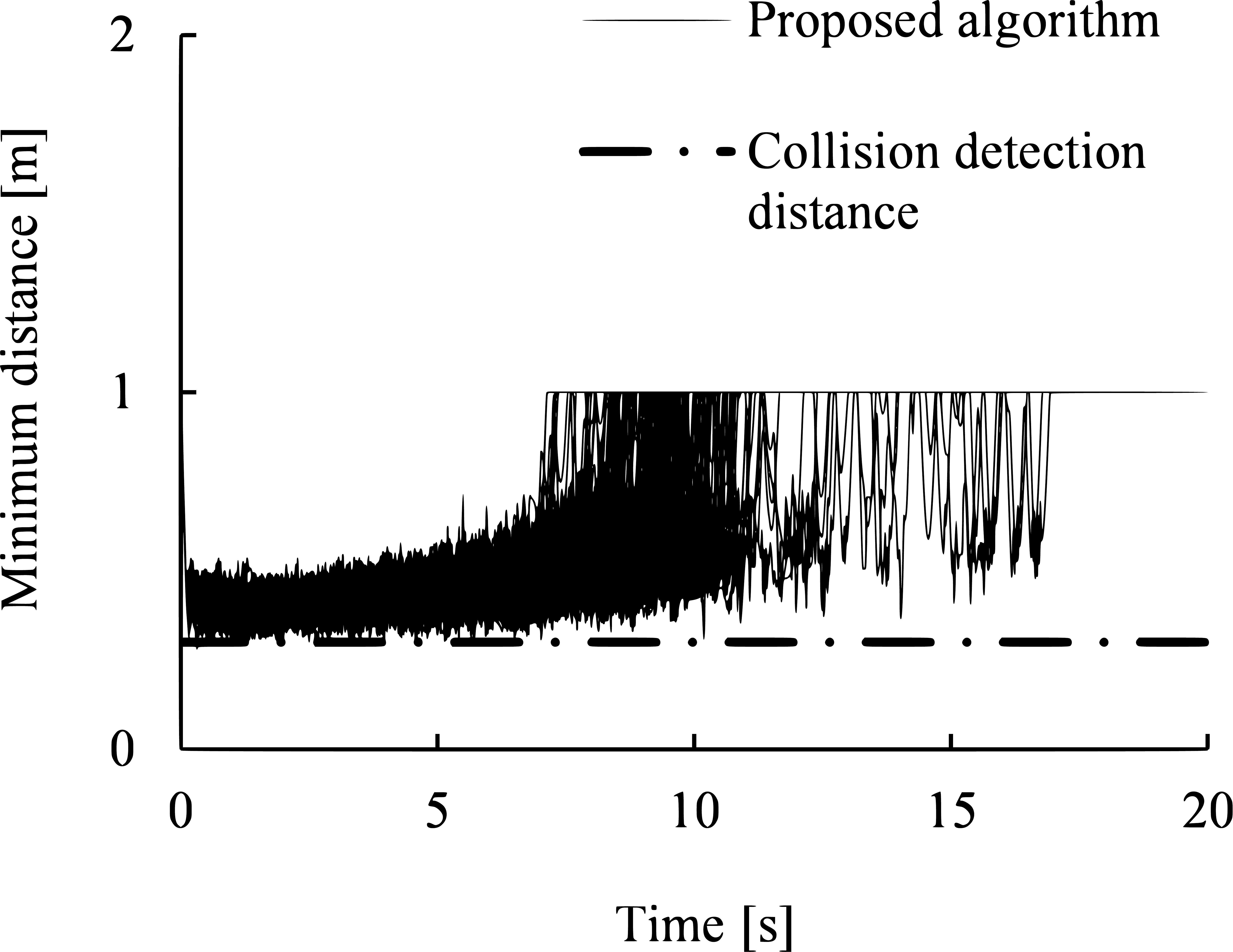}
\label{fig5b}}
\caption{Time-series of minimum distances among 100 agents covering 100 simulations.}
\label{fig5}
\end{figure}

In addition, to confirm the safety of proposed collision avoidance method, the minimum distances between agents are shown in figure \ref{fig5}.
From the results, the minimum distances between agents are mostly longer than the collision detection distance.


\section{Experiment}
In this section, we will demonstrate the use of the proposed algorithm through the experiment with switching positions of targets using eight aerial robots. 
We will first explain the experimental configuration used to control multiple aerial robots, and then show the experimental conditions and results. 
Based on those results, we confirmed the global optimal deployment, collision safety and trajectory smoothness under the influences of wind disturbance by aerial robots and the position control errors. 

\subsection{Experimental configuration and conditions}
The Robot Operating System (ROS) was utilized to realize coverage control system by multiple robots. 
As for the robots, we used Bebop 2 drones (Parrot S.A., Paris, France), whose control package is available as an open-source software package in \cite{19}.


The positions of robots were obtained using an OptiTrack Prime 17W motion capture system (Natural Point, Inc., Corvallis, OR) operating at 120 Hz. 
The measurement data were sent a control personal computer (PC) with 8-core Intel$^{\textregistered}$ Core$^{\texttrademark}$ i7 (2.80GHz), 32 GB of RAM, by using the Virtual-Reality Peripheral Network (VRPN) protocol and the ROS node client, as described in \cite{20}. 

Using positions of neighboring agents, each robot calculates index set of neighborhood agents, index set of assigned targets, gets the coverage rate of each target from neighborhood agents, and then calculates control input using (\ref{eq7}) and (\ref{eq10}). 
Since the dynamics of the robot follows the dynamics of the three-dimensional rigid body, each aerial robot is individually controlled to move to its target position using the individual controller as described in \cite{21}.
In more detail, the individual control is constructed using position controller with P controller, velocity controller with PI controller, attitude controller with inner model controller considering dead time of the control input. 
Control inputs calculated in the control PC were sent via the router by constructing a communication system based on \cite{22} in which all the robots could receive their control inputs. 

\begin{table}[!tp]
\tbl{Experimental conditions}
{\begin{tabular}[l]{@{}lccc}\toprule
  Variable & Symbol & Value\\
\colrule
  Number of robots & $n$ & 8  \\
  Number of targets & $n_t$ & 8  \\
  Coverage region size [m]$\times$[m] & - & 10.0$\times$5.0  \\  
  Time step[s] & $\Delta t$ & 0.02  \\  
  Max speed of robot $i\in \nu$ [m/s] & - & 1.5  \\  
  Visual sensor range of robot $i\in \nu$ [m]$\times$[m] & $D({\textit{\textbf{x}}_i})$ & 1.0$\times$1.0  \\  
  Communication range [m] & $d_c$ & 5.0  \\  
  Robot size [m]$\times$[m] & - & 0.3$\times$0.3  \\  
  Collision detection distance [m] & - & 0.3  \\  
  Distance to activate collision avoidance [m] & $d_k$ & 1.0 \\     
  Parameters for collision avoidance & $K_d$ & 50.0  \\ 
  Parameters for collision avoidance & $K_s$ & 0.35  \\       
  Propotional gain of position controller \cite{21}& - & 1.0  \\    
  Propotional gain of velocity controller \cite{21}& - & 0.15  \\        
  Integral gain of velocity controller \cite{21}& - & 0.04  \\                 
\botrule
\end{tabular}}
\label{symbols}
\end{table}

The proposed algorithm was confirmed in an indoor experiment with the experimental conditions in table 2. 
The positions of robots were controlled by the control PC at 50 Hz in order to consider the communication time between the PC and the robots. The coverage region $Q$ was given by
$$
q\in Q:=\{(x,y)\mid -5\le x\le5, -2.5\le y\le2.5 \},
\eqno{}
$$
where $x$- and $y$-axis represent the longitudinal and lateral axes, respectively. In order to demonstrate the effectiveness of the algorithm, we consider 12 cases with different target positions to be changed in order.

\subsection{Experimental results}
Sequential shots of four different coverage problems are shown in figure \ref{fig6}.
The others are omitted due to space limitations. Top views of the deployments obtained are shown in figure. \ref{fig7}. 
From these results, it can be seen that all the targets were equally divided into each Voronoi region and captured within the visual sensor ranges of the eight robots. 
Additionally, the aerial robot trajectories show that the robots could reach their target positions smoothly.

In order to confirm the global optimality, the coverage rates for 12 experiments are shown in figure \ref{fig8}. From the results obtained, we can see that the total coverage rate converged to 100 $\%$ in all cases by applying the proposed algorithm.

Moreover, figure \ref{fig9} shows time series of minimum distances between different robots. From the results, the minimum distances were always longer than the collision detection distance for all the cases.

\begin{figure}[!tp]
\centering
\subfigure[Case. 1: from dot targets to cross targets]{
\includegraphics[width=6cm]{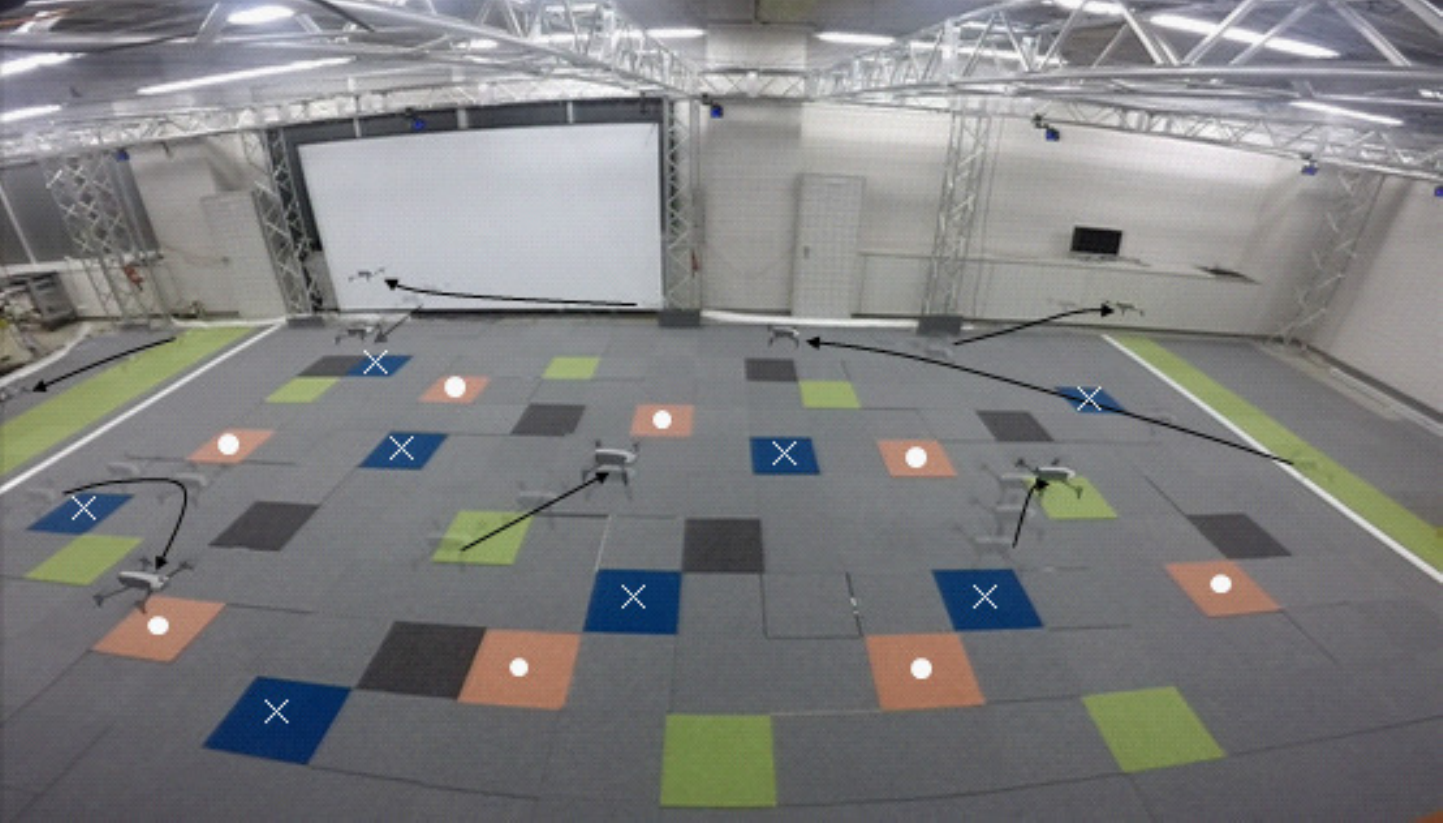}
\label{case1}}
\subfigure[Case. 2: from cross targets to triangle targets]{
\includegraphics[width=6cm]{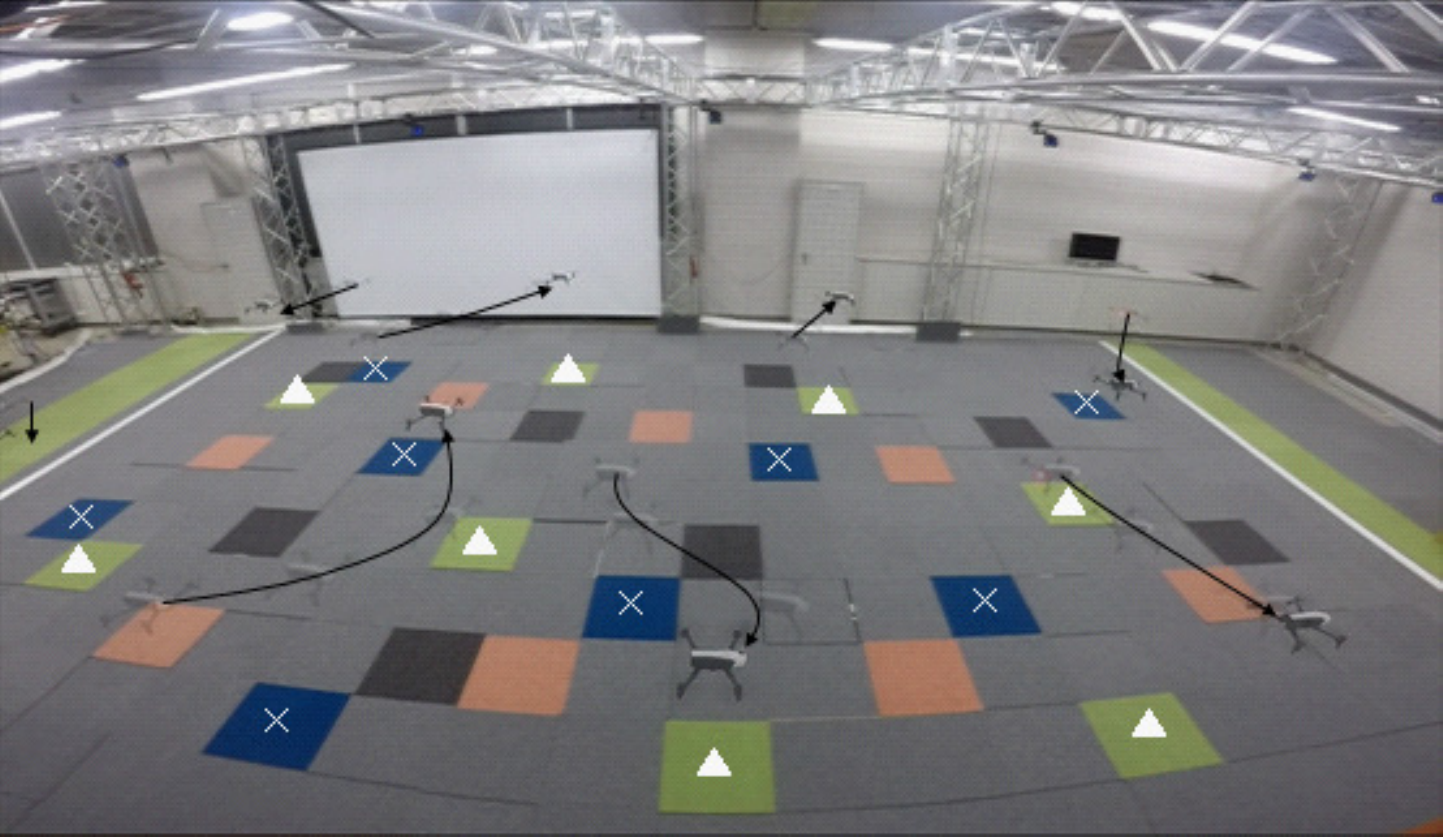}
\label{case2}}
\subfigure[Case. 3: from triangle targets to square targets]{
\includegraphics[width=6cm]{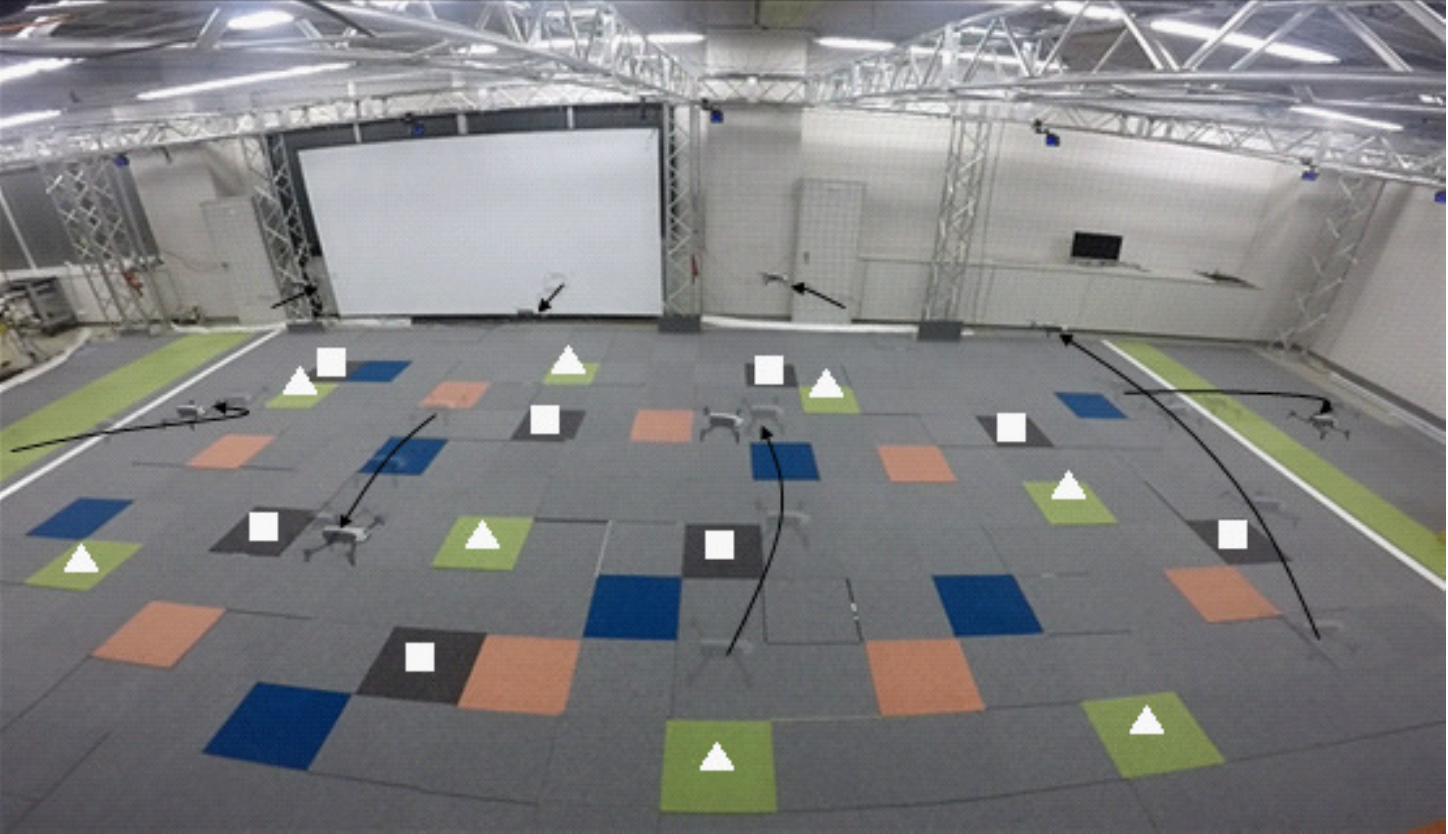}
\label{case3}}
\subfigure[Case. 4: from square targets to dot targets]{
\includegraphics[width=6cm]{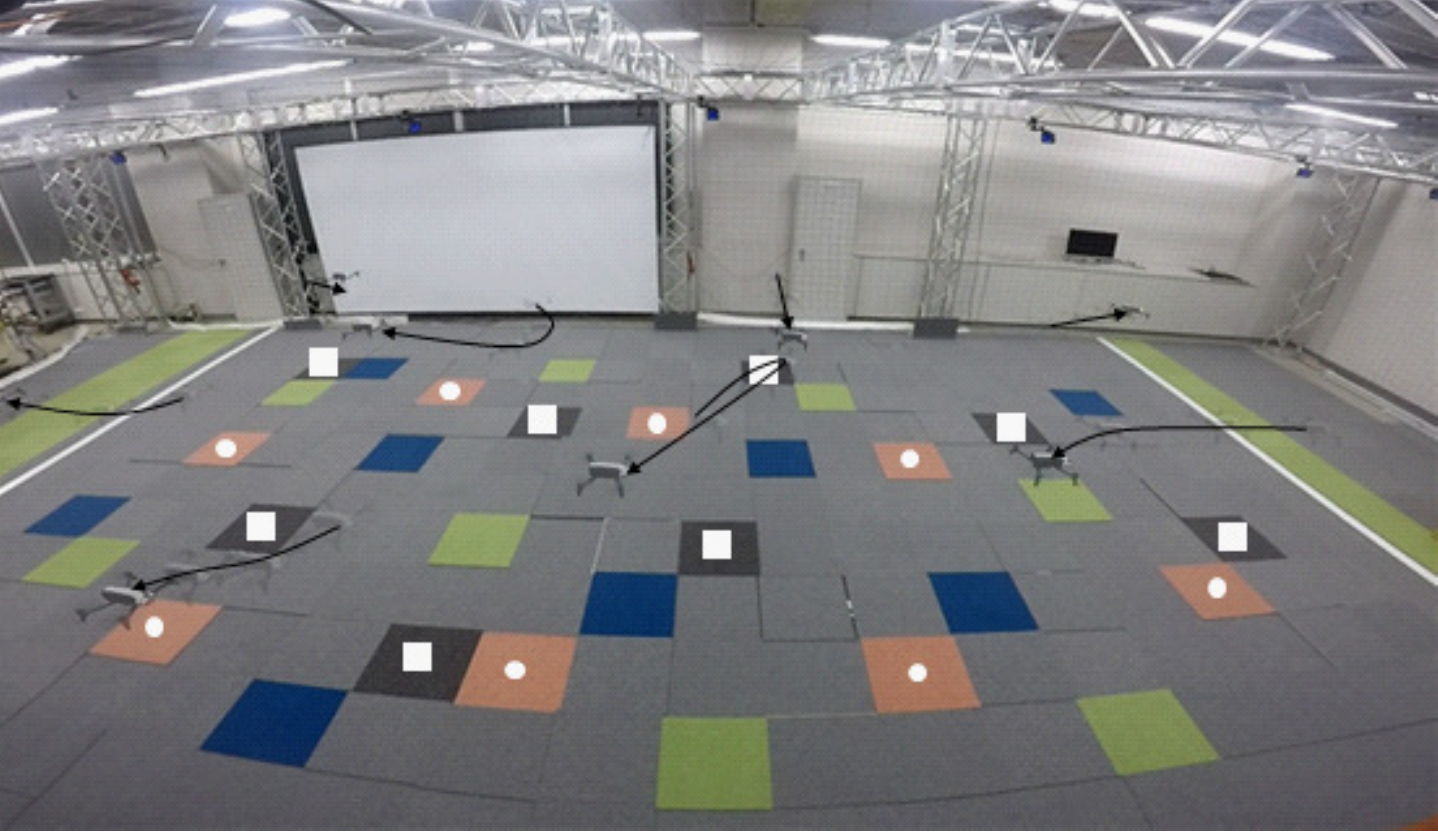}
\label{case4}}
\caption{Four sequential shots of the four cases of target changes in the indoor experiment by applying the proposed algorithm. The colored areas represent targets. Arrows show the trajectories of the robots.}
\label{fig6}
\end{figure}

\begin{figure}[!tp]
\centering
\subfigure[Case. 1]{
\includegraphics[width=6cm]{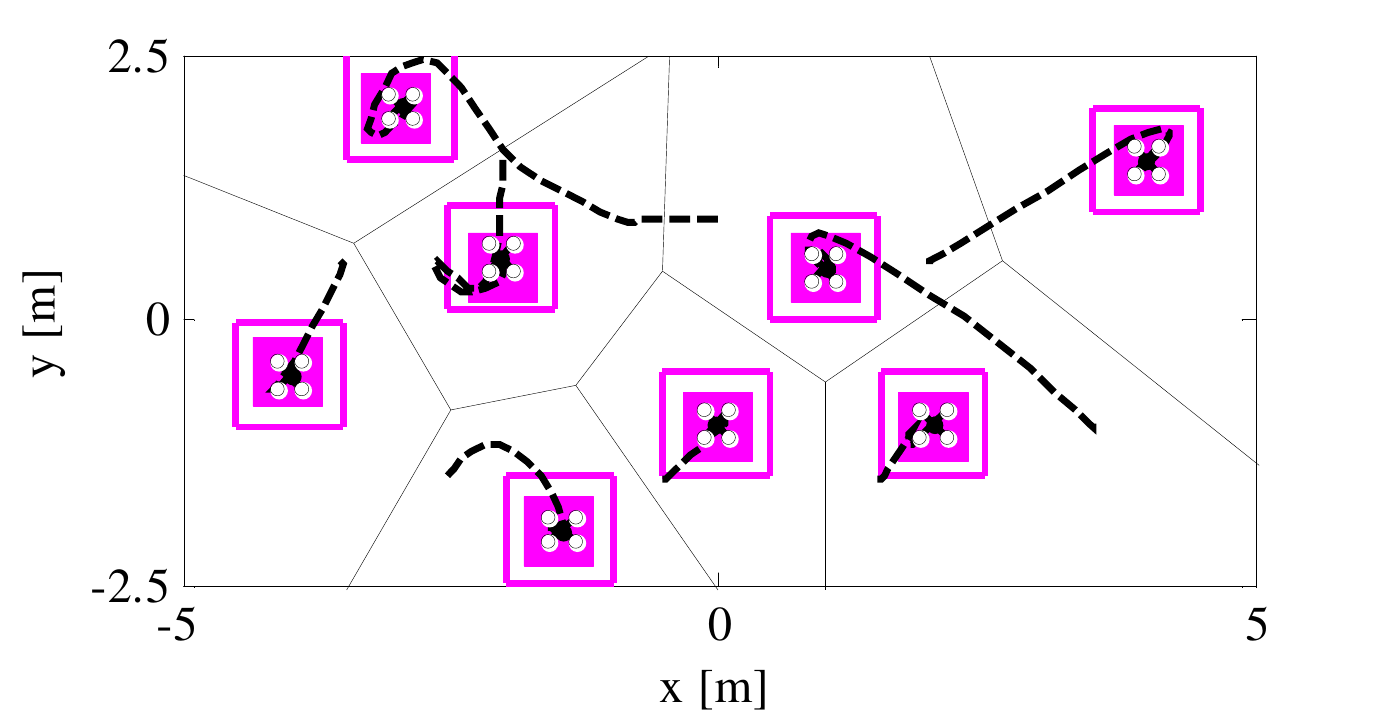}
\label{case1}}
\subfigure[Case. 2]{
\includegraphics[width=6cm]{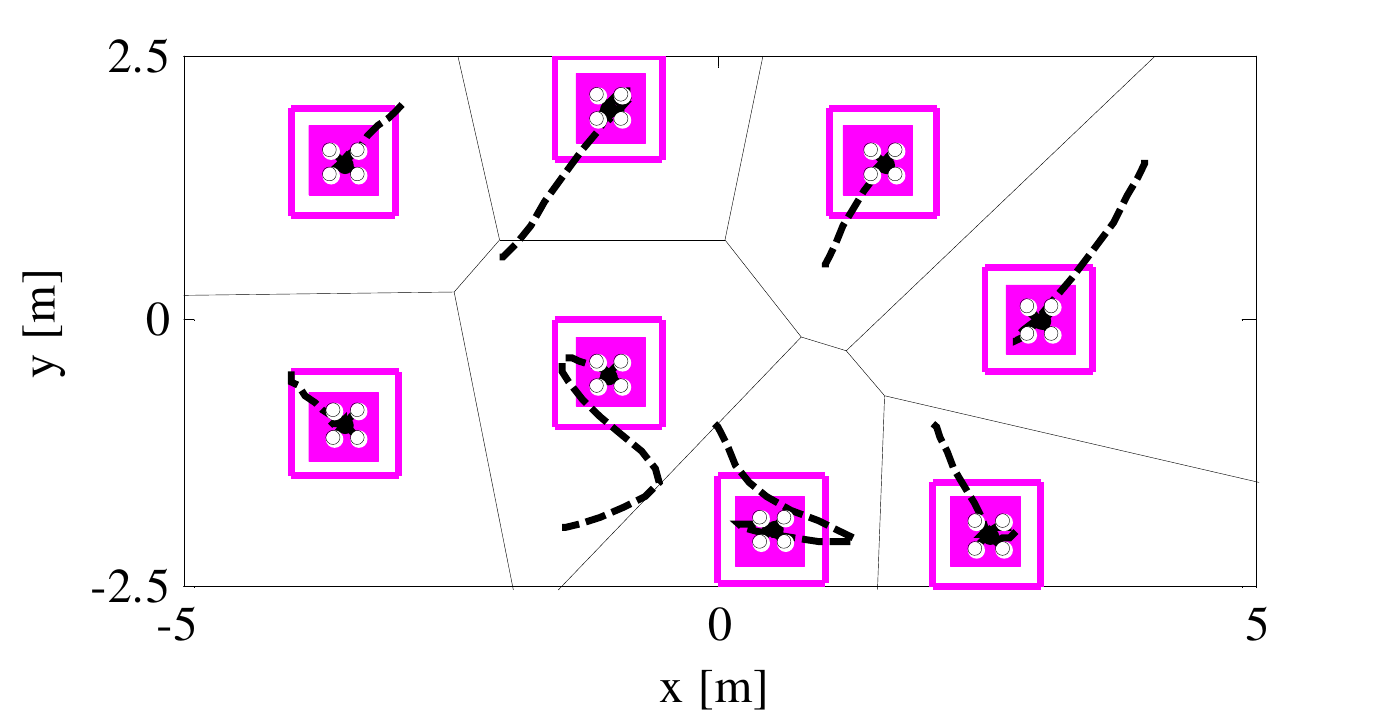}
\label{case2}}
\subfigure[Case. 3]{
\includegraphics[width=6cm]{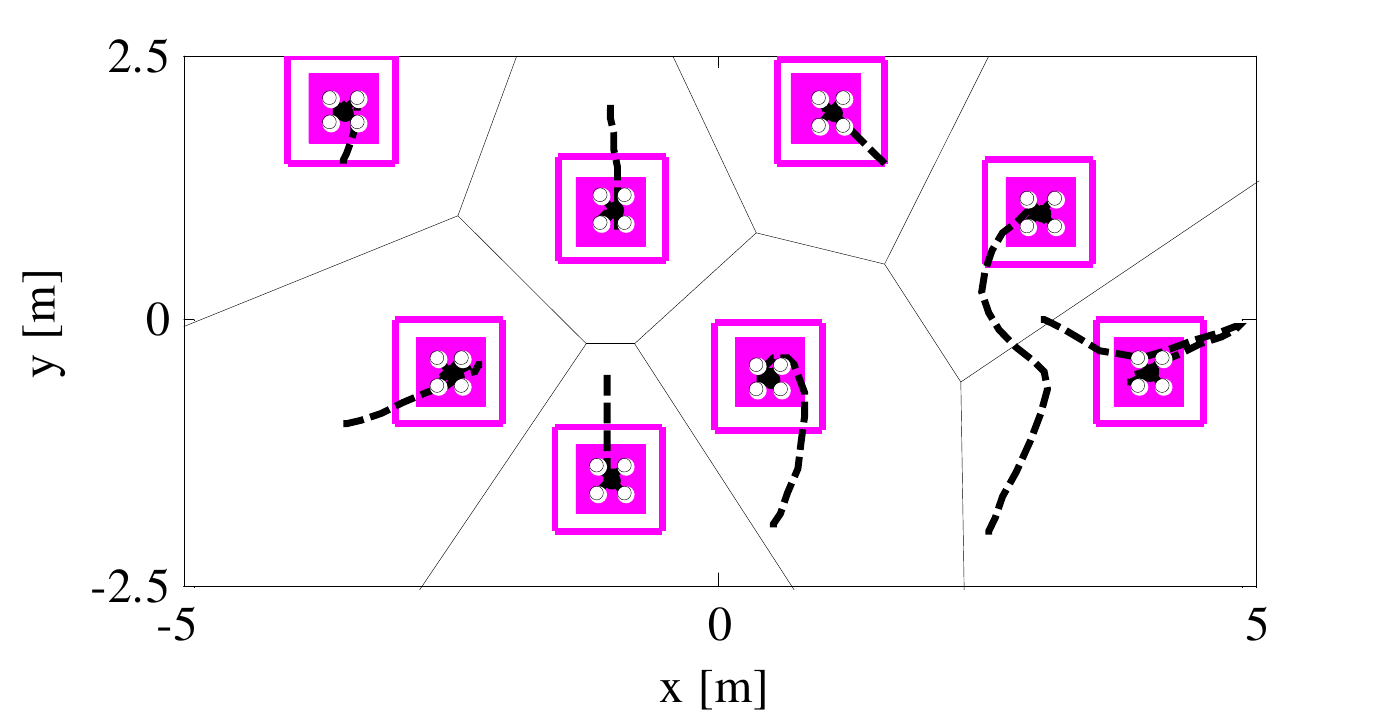}
\label{case3}}
\subfigure[Case. 4]{
\includegraphics[width=6cm]{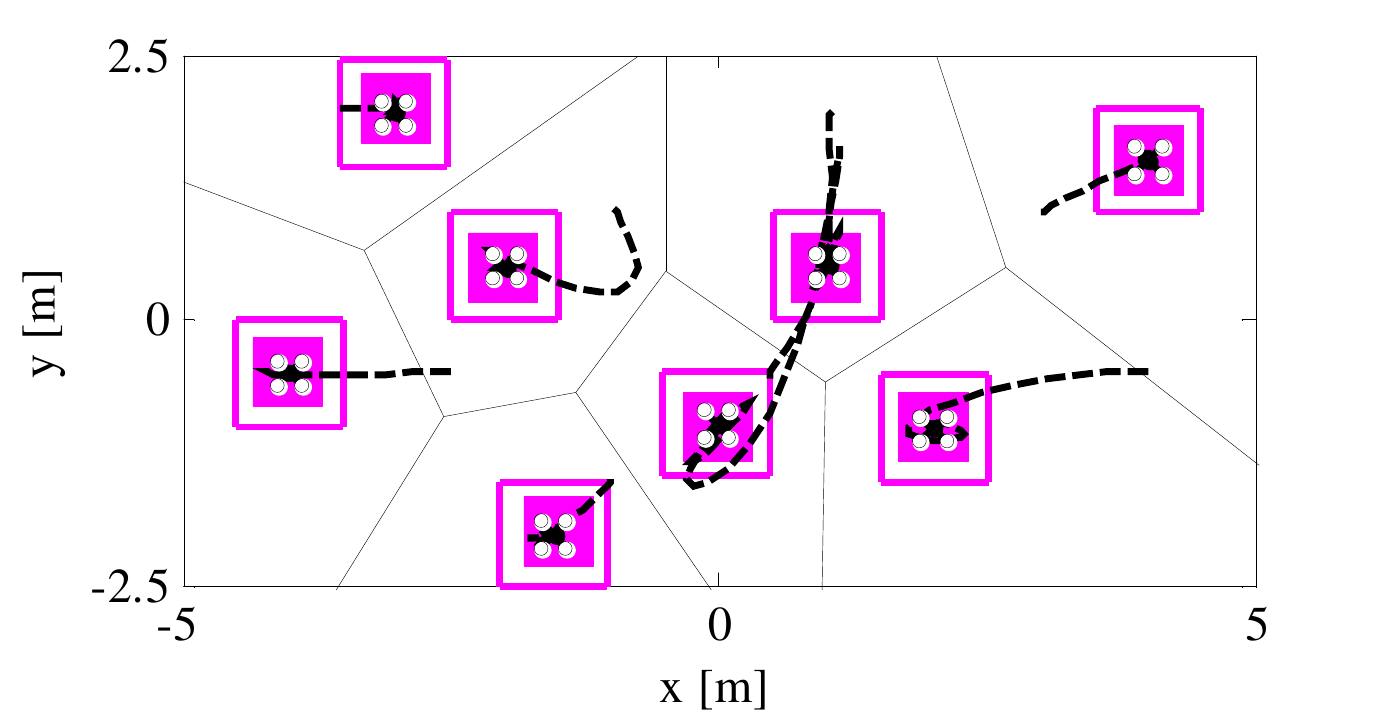}
\label{case4}}
\caption{Deployment changes of 8 robots by applying the proposed algorithm. Solid lines, square frames, filled squares, and dashed lines represent Voronoi partitions, sensor ranges, targets, trajectories of robots in the latest 10 s, respectively.}
\label{fig7}
\end{figure}

\begin{figure}[!tp]
\vspace{4mm}
\begin{center}
\includegraphics[width=9cm]{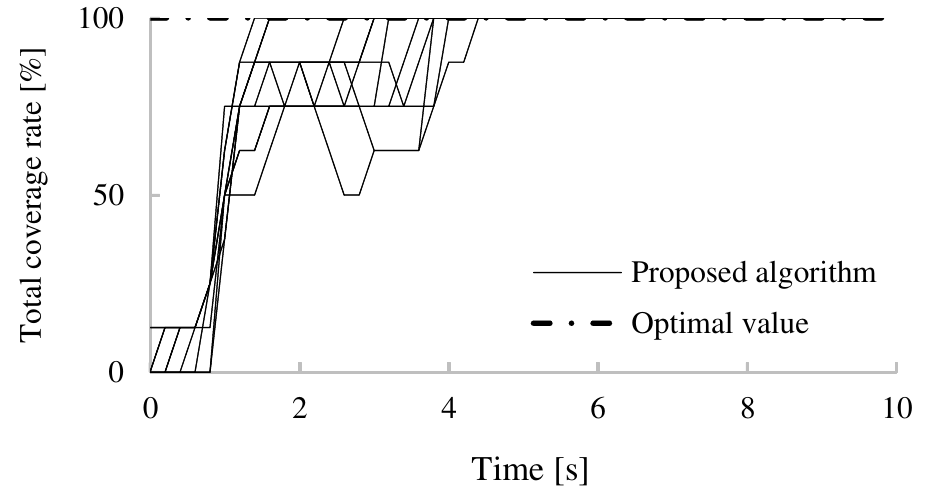}
\caption{Total coverage rates for 12 cases in the indoor experiment.}
\label{fig8}
\end{center}
\end{figure}

\begin{figure}[!tp]
\vspace{4mm}
\begin{center}
\includegraphics[width=9cm]{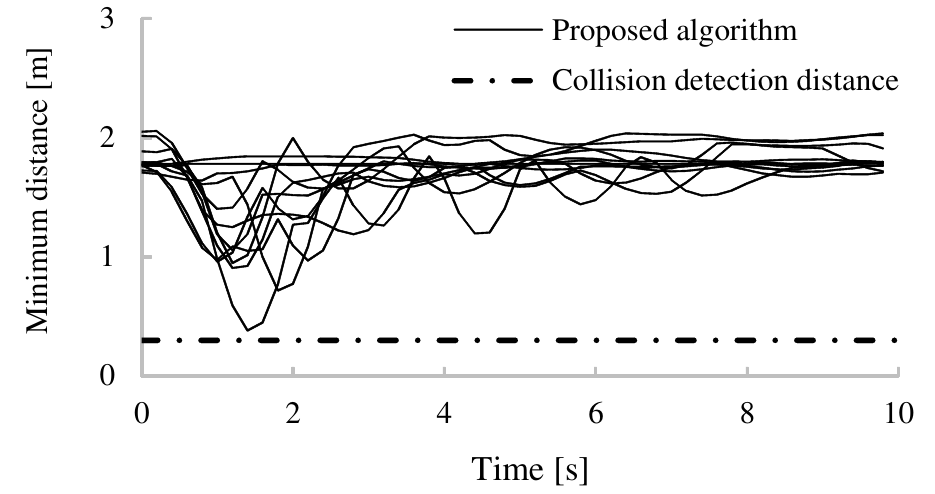}
\caption{Time-series of minimum distances for 12 cases in the indoor experiment.}
\label{fig9}
\end{center}
\end{figure}

\subsection{Discussion}

In this section, we discuss the remaining challenges and the future potential of the proposed algorithm.

In the current system, the positions of robots are obtained using a motion capture system.
Meanwhile, in a purely distributed system, each robot should estimate its own position using onboard sensors.
We should confirm the effectiveness of our algorithm under the influences of the uncertainties in the estimated positions.

Moreover, our algorithm requires some improvements when applying the coverage problem in three-dimensional space.
Cut-in algorithm in (\ref{eq8}), (\ref{eq9}) and (\ref{eq10}) is applicable regardless of the dimension.
However, the target index set $C_i$ in (\ref{eq4}) should be modified considering not only the positions of robots but also directions of cameras as described in \cite{8}.

Despite these unsolved challenges, the proposed algorithm has a potential to be promising solution for stable operation of the large-scale monitoring systems.
Even if some robots did not work owing to fuel exhaustion or sensor failure, our algorithm enabled remaining sound robots to maximize the total coverage performance.

\section{Conclusion}
In this paper, we propose a global optimal coverage control algorithm with collision avoidance that enables multiple robots to move smoothly in continuous space and deterministically converge to a global optimal deployment. 
We prove that the robots deterministically moves to the global optimal deployment by applying our coverage control algorithm without considering collision avoidance. 
In addition, we integrate a modified potential method into the algorithm for collision avoidance, and then show that the proposed algorithm can derive global optimal deployments with a sufficient level of collision safety through numerical simulations with randomly arranged targets. 
Finally, we demonstrated the effectiveness of the proposed algorithm by experiments using multiple aerial robots and confirmed the global optimal deployment, collision safety and trajectory smoothness even under the influences of wind disturbance by aerial robots and the position control errors. 
The limitation of the current system is the use of a motion capture system, so we will implement localization using onboard cameras.
Moreover, we will modify our algorithm for the coverage problem in three-dimensional space.
In the future, we plan to design the convergence speed to the global optimal deployment in related to the energy consumption.

\section*{Acknowledgement}
This is an original manuscript of an article published by Taylor \& Francis in ADVANCED ROBOTICS on July 8, 2019, available online at: http://www.tandfonline.com/doi/10.1080/01691864.2019.1637777.

\end{document}